\documentclass[11pt]{article}
\usepackage{fullpage,times,url,bm}

\usepackage{amsthm,amsfonts,amsmath,amssymb,epsfig,color,float,graphicx,verbatim}
\usepackage{algorithm,algorithmic}
\usepackage{bbm}

\usepackage[numbers]{natbib}
\usepackage[pdftex,colorlinks]{hyperref}

\usepackage{hyperref}
\hypersetup{
	colorlinks   = true, %Colours links instead of ugly boxes
	urlcolor     = blue, %Colour for external hyperlinks
	linkcolor    = blue, %Colour of internal links
	citecolor   = black %Colour of citations
}

\usepackage{bbm}
\newtheorem{theorem}{Theorem}[section]

\newtheorem{lemma}[theorem]{Lemma}
\newtheorem{corollary}[theorem]{Corollary}

\newcommand{\reals}{\mathbb{R}}

\newcommand{\sphere}{\mathbb{S}}

\newcommand{\E}{\mathbb{E}}

\newcommand{\var}{\text{Var}}

\newcommand{\bx}{\mathbf{x}}
\newcommand{\bw}{\mathbf{w}}

\newcommand{\bv}{\mathbf{v}}

\newcommand{\by}{\mathbf{y}}

\newcommand{\Acal}{\mathcal{A}}

\newcommand{\Dcal}{\mathcal{D}}

\newcommand{\Ncal}{\mathcal{N}}

\newcommand{\inner}[1]{\langle#1\rangle}

\newtheorem{claim}{Claim}

\renewcommand{\eqref}[1]{Eq.~(\ref{#1})}

\makeatletter
\newcommand{\printfnsymbol}[1]{%
  \textsuperscript{\@fnsymbol{#1}}%
}
\makeatother



\author{
{\small Amit Daniely}\\
{\small Hebrew University and Google}\\
{\small \texttt{amit.daniely@mail.huji.ac.il }}\\
\and
{\small Nathan Srebro}\\
{\small TTI-Chicago}\\ 
{\small \texttt{nati@ttic.edu}}\\
\and
{\small Gal Vardi}\\
{\small TTI-Chicago and Hebrew University}\\ 
{\small \texttt{galvardi@ttic.edu}}\\
\and
{\small Collaboration on the Theoretical Foundations of Deep Learning (\url{deepfoundations.ai})}
}

\title{Most Neural Networks Are Almost Learnable}
\date{}
\begin{document}

\maketitle

\begin{abstract}
    We present a PTAS for learning random constant-depth networks.
    We show that for any fixed $\epsilon>0$ and depth $i$, there is a poly-time algorithm that for any distribution on $\sqrt{d} \cdot \sphere^{d-1}$ learns random Xavier networks of depth $i$, up to an additive error of $\epsilon$. The algorithm runs in time and sample complexity of $(\bar{d})^{\mathrm{poly}(\epsilon^{-1})}$, where $\bar d$ is the size of the network. For some  cases of sigmoid and ReLU-like activations the bound can be improved to $(\bar{d})^{\mathrm{polylog}(\epsilon^{-1})}$, resulting in a quasi-poly-time algorithm for learning constant depth random networks.
\end{abstract}

\section{Introduction}

One of the greatest mysteries surrounding deep learning is the 
%great 
discrepancy between its phenomenal capabilities in practice 
and the fact that despite a great deal of research, polynomial-time algorithms for learning deep models are known only for very restrictive cases. Indeed, state of the art results are only capable of dealing with two-layer networks under assumptions on the input distribution and the network's weights. 
Furthermore, theoretical study shows that even with very naive architectures, learning neural networks is worst-case computationally intractable. 

In this paper, we contrast the aforementioned theoretical state of affairs, and show that, perhaps surprisingly, even though constant-depth networks are completely out of reach from a worst-case perspective, \emph{most} of them are not as hard as one would imagine. That is, they are {\em distribution-free learnable} in polynomial time up to any desired constant accuracy.
%Understanding the computational complexity of learning neural networks is a fundamental question in the theory of deep learning. %and it has been the subject of extensive research in recent years.  
%Despite a great deal of research, polynomial-time algorithms for learning neural networks are known only for two-layer networks under certain assumptions on the input distribution and the network's weights. 
%Moreover, learning networks of depth greater than $2$ is hard (under cryptographic assumptions), even if the input distribution is Gaussian and the weight matrices are non-degenerate \citep{daniely2023computational,chen2022hardness,daniely2021local}.
%In this work, we study the computational complexity of learning neural networks using an average-case approach. 
%Perhaps surprisingly, we show that learning a random neural network of constant depth and Lipschitz activations within an arbitrarily small constant accuracy can be done in polynomial time. 
%The random networks are obtained using the standard Xavier initialization scheme \cite{glorot2010understanding,he2015delving}  and our result holds for every input distribution supported on the sphere $\sqrt{d} \cdot \sphere^{d-1}$. 
%Thus, most constant-depth neural networks can be learned efficiently within a constant accuracy. 
%Thus, \emph{most} constant-depth neural networks can be learned with an additive polynomial-time approximation scheme (PTAS). 
This is the first polynomial-time approximation scheme (PTAS) for learning neural networks of depth greater than $2$ (see the related work section for more details).
Moreover, we show that the standard SGD algorithm on a ReLU network can be used as a PTAS for learning random networks.
The question of whether learning random networks can be done efficiently was posed by \citet{daniely2023computational}, and our work provides a positive result in that respect.

In a bit more detail, we consider constant-depth random networks obtained using the standard Xavier initialization scheme \cite{glorot2010understanding,he2015delving}, and any input distribution supported on the sphere $\sqrt{d} \cdot \sphere^{d-1}$. For Lipschitz activation functions, our algorithm runs in time $(\bar{d})^{\mathrm{poly}(\epsilon^{-1})}$, where $\bar{d}$ is the network's size including the $d$ input components, and $\epsilon$ is the desired accuracy. While this complexity is polynomial for constant $\epsilon$, we also consider the special cases of sigmoid and ReLU-like activations, where the bound can be improved to $(\bar{d})^{\mathrm{polylog}(\epsilon^{-1})}$.
%, namely, for $\epsilon=1/\mathrm{poly}(\bar{d})$ the complexity is quasi-polynomial in $\bar{d}$.

The main technical idea in our work is that constant-depth random neural networks
% (obtained using the standard Xavier initialization scheme \cite{glorot2010understanding,he2015delving}) 
with Lipschitz activations can be approximated sufficiently well by low-degree polynomials. This result follows by analyzing the network obtained by replacing each activation function with its polynomial approximation using Hermite polynomials. It implies that efficient algorithms for learning polynomials can be used for learning random neural networks, and specifically that we can use the SGD algorithm on ReLU networks for this task.

%Our result demonstrates that average-case complexity analyses w.r.t. the network weights can be a 
%useful approach for
%obtaining provable polynomial-time guarantees for learning neural networks. 
%The average-case approach might also help resolve the apparent contradiction between the practical success of learning neural networks and the strong hardness results.

\subsection{Results}

In this work, we show that random fully-connected feedforward neural networks can be well-approximated by low-degree polynomials, which implies a PTAS for learning random networks. We start by defining the network architecture. We will denote by $\sigma:\reals\to\reals$ the activation function, and will 
assume that it is $L$-Lipschitz. To simplify the presentation, we will also assume that it is normalized in the sense that $\E_{X\sim \Ncal(0,1)}\sigma^2(X)=1$. Define
$\epsilon_\sigma(n) = \min_{\deg(p)=n}\E_{X\sim\Ncal(0,1)}(\sigma(X)-p(X))^2$, namely, the error when approximating $\sigma$ with a degree-$n$ polynomial, and note that $\lim_{n\to\infty}\epsilon_\sigma(n) =0$.
We will consider fully 
connected networks of depth $i$ and will use $d_0 = d$ to denote 
the input dimension and $d_1,\ldots,d_i$ to denote the number of neurons in each layer. Denote also  $\bar d = \sum_{j=0}^id_j$.
%Throughout, we will consider $\sigma$ as {\em constant}, and will allow other constants to depend on $\sigma$.
Given weight matrices
\[
\vec{W} = (W^1,\ldots, W^i)\in \reals^{d_1\times d_0}\times\ldots\times\reals^{d_i\times d_{i-1}}
\]
and $\bx\in\reals^{d_0}$ we define $\Psi_{\vec{W}}^0(\bx) = \bx$. Then for $1\le j\le i$ we define recursively
\[
\Phi^{j}_{\Vec{W}}(\bx) = W^j\Psi^{j-1}_{\Vec{W}}(\bx),\;\;\;\;
\Psi^{j}_{\Vec{W}}(\bx) = \sigma\left(\Phi^{j}_{\Vec{W}}(\bx)\right)
\]
We will consider random networks in which the weight matrices are random {\em Xavier matrices~\cite{glorot2010understanding,he2015delving}}. That is, each entry in $W^j$ is a centered Gaussian of variance $\frac{1}{d_{j-1}}$. This choice is motivated by the fact that it is a standard practice to initialize the network's weights with Xavier matrices, and furthermore, it ensures that the scale across the network is the same. That is, for any example $\bx$ and a neuron $n$,
the second moment of the output of $n$ (w.r.t. the choice of $\vec{W}$) is $1$.

Our main result shows that $\Psi_{\vec{W}}^i$ can be approximated, up to any constant accuracy $\epsilon$, via constant degree polynomials (the constant will depend only on $\epsilon$, the depth $i$, and the activation $\sigma$). 
We will consider the input space $\tilde\sphere^{d-1} = \{\bx\in\reals^{d} : \|\bx\| = 1\}$. Here, and throughout the paper, $\|\bx\|$ stands for the {\em normalized} Euclidean norm $\|\bx\| = \sqrt{\frac{1}{d}\sum_{i=1}^{d}x_i^2}$.
\begin{theorem}\label{thm:main_intro}
    For every $i$ and $n$ such that $\epsilon_\sigma(n)\le \frac{1}{2}$ there is a constant $D = D(n,i,\sigma)$ such that 
    if $d_1,\ldots,d_{i-1} \ge D$ the following holds. For any weights $\vec{W}$, there is a degree $n^{i-1}$ polynomial $p_{\vec{W}}$ such that for any distribution $\Dcal$ on $\tilde\sphere^{d-1}$
    \[
    \E_{\vec{W}}\E_{\bx\sim\Dcal}\left\| \Phi_{\vec{W}}^i(\bx)- p_{\vec{W}}(\bx)  \right\| \le 14 \cdot (L+1)^2 \cdot \left(\epsilon_\sigma(n)\right)^{\frac{1}{2^{i-1}}} \le \frac{14 \cdot (L + 1)^3 }{n^{\frac{1}{2^{i-1}}}}~.
    \]
    Furthermore, the coefficients of $p_{\vec{W}}$ are bounded by $(2\bar d)^{4n^{i-1}}$.
\end{theorem}
Since constant degree polynomials are learnable in polynomial time, Theorem \ref{thm:main_intro} implies a PTAS for learning random networks of constant depth. In fact, as shown in~\cite{daniely2017sgd}, constant degree polynomials with polynomial coefficients are efficiently learnable via SGD on ReLU networks starting from standard Xavier initialization. Thus, this PTAS can be standard SGD on neural networks.
To be more specific, for any constant $\epsilon>0$ there is an algorithm with $(\bar d)^{O\left(\left(\frac{14(L+1)^3}{\epsilon}\right)^{(i-1)2^{i-1}}\right)}$ time and sample complexity that is guaranteed to return a hypothesis whose  loss is at most $\epsilon$ in expectation.
For some specific activations, such as the sigmoid $\sigma(x) = \mathrm{erf}(x) := \frac{2}{\sqrt{\pi}}\int_0^{x}e^{-\frac{t^2}{2}}dt$, or the ReLU-like activation $\sigma(x)=\int_{0}^{x}\mathrm{erf}(t) + 1dt $ we have that $\epsilon_\sigma(n)$ approaches to $0$ exponentially fast (see Lemma~\ref{lem:eps_n_sig}). In this case, we get 
get a {\em quasi-polynomial} time and sample complexity of $(\bar d)^{O\left(\left(\log\left(\frac{14(L+1)^3}{\epsilon}\right)\right)^{(i-1)}\right)}$.

\begin{corollary}\label{cor:main}
    For every constants $\epsilon, i$ and $\sigma$ there is a constant $D$, a univariate-polynomial $p$ and a polynomial-time algorithm $\Acal$ such that 
    if $d_1,\ldots,d_{i-1} \ge D$ the following holds. For any distribution $\Dcal$ on $\tilde\sphere^{d-1}$, if $h$ is the output of $\Acal$ upon seeing $p(d_0,\ldots,d_i)$ examples from $\Dcal$, then\footnote{The leftmost expectation denoted $\E_h$ is over the examples provided to $\Acal$, as well as the internal randomness of $\Acal$.}
    \[
    \E_{h}\E_{\vec{W}}\E_{\bx\sim\Dcal}\left\| \Phi_{\vec{W}}^i(\bx)- h(\bx)  \right\| \le \epsilon~.
    \]
    Furthermore, $\Acal$ can be taken to be SGD on a ReLU network starting from a Xavier initialization.
\end{corollary}

\subsection{Related work}

\paragraph{Learning neural networks efficiently.}

Efficiently learning classes of neural networks has attracted much interest in recent years. 
Several works established polynomial-time algorithms for learning one-hidden-layer neural networks with certain input distributions (such as the Gaussian distribution) under the assumption that the weight matrix of the hidden layer is non-degenerate \citep{janzamin2015beating,zhong2017recovery,ge2017earning,ge2018learning,bakshi2019learning,zhang2019learning,awasthi2021efficient}.
For example, \citet{awasthi2021efficient} showed 
%that one-hidden-layer ReLU networks can be learned efficiently if the input distribution is Gaussian and the weight matrix of the hidden layer is non-degenerate, 
such a result for non-degenerate one-hidden-layer ReLU networks with bias terms under Gaussian inputs,
and also concluded that 
%these networks
one-hidden-layer 
%ReLU networks with Gaussian inputs 
networks
can be learned efficiently under the smoothed-analysis framework.  
%
%In \cite{diakonikolas2020algorithms}, the authors showed an 
Efficient algorithms for learning one-hidden-layer ReLU networks with Gaussian inputs were also shown in  \citet{diakonikolas2020algorithms,diakonikolas2020small}. These results do not require non-degenerate weight matrices, but they require that the output layer weights are all positive, as well as a sub-linear upper bound on the number of hidden neurons.
\citet{chen2023learning} recently showed an efficient algorithm for learning one-hidden-layer ReLU networks with Gaussian inputs, under the assumption that the number of hidden neurons is a constant.
Note that all of the aforementioned works consider only one-hidden-layer networks. 
\citet{chen2022learning} gave an algorithm for learning deeper ReLU networks, whose complexity is polynomial in the input dimension but exponential in the other parameters (such as the number of hidden units, depth, spectral norm of the weight matrices, and Lipschitz constant of the overall network).
Finally, several works established algorithms for learning neural networks, whose complexity is exponential unless we impose strong assumptions on the norms of both the inputs and the weights \citep{goel2019learning,vempala2019gradient,zhang2016l1,goel2017reliably}.

\paragraph{Hardness of learning neural networks.}

As we discussed in the previous paragraph, efficient algorithms for learning ReLU networks are known only for depth-$2$ networks and under certain assumptions on both the network weights and the input distribution. The limited progress 
%in this direction 
in learning ReLU networks
can be partially understood by an abundance of hardness results. 

Learning neural networks without any assumptions on the input distribution or the weights is known to be hard (under cryptographic and average-case hardness assumptions) already for depth-$2$ ReLU networks \citep{KlivansSh06,applebaum2010public,daniely2016complexity}. 
For depth-$3$ networks, hardness results were obtained already when the input distribution is Gaussian \citep{daniely2021local,chen2022hardness}. 
All of the aforementioned hardness results are for improper learning, namely, they do not impose any restrictions on the learning algorithm or on the hypothesis that it returns. 
For \emph{statistical query (SQ)} algorithms, unconditional superpolynomial lower bounds were obtained for learning depth-$3$ networks with Gaussian inputs \citep{chen2022hardness}, and superpolynomial lower bounds for \emph{Correlational SQ (CSQ)} algorithms were obtained already for learning depth-$2$ networks with Gaussian inputs \citep{goel2020superpolynomial,diakonikolas2020algorithms}.

The above negative results suggest that assumptions on the input distribution may not suffice for obtaining efficient learning algorithms. Since in one-hidden-layer networks efficient algorithms exist when imposing assumptions on both the input distribution and the weights, a natural question is whether this approach might also work for deeper networks. Recently, \citet{daniely2023computational} gave a hardness result for improperly learning depth-$3$ ReLU networks under the Gaussian distribution even when the weight matrices are non-degenerate. This result suggests that learning networks of depth larger than $2$ might require new approaches and new assumptions. Moreover, \cite{daniely2023computational} showed hardness of learning depth-$3$ networks under the Gaussian distribution even when a small random perturbation is added to the network's parameters, namely, they proved hardness in the smoothed-analysis framework. While adding a small random perturbation to the parameters does not seem to make the problem computationally easier, they posed the question of whether learning random networks, which roughly correspond to adding a large random perturbation, can be done efficiently. The current work gives a positive result in that respect.

\citet{daniely2020hardness} studied whether there exist some ``natural" properties of the network's weights that may suffice to allow efficient distribution-free learning, where a ``natural" property is any property that holds w.h.p. in random networks. More precisely, they considered a setting where the target network is random, an adversary chooses some input distribution (that may depend on the target network), and the learning algorithm needs to learn the random target network under this input distribution. They gave a hardness result for improper learning (within constant accuracy) in this setting. Thus, they showed that learning random networks is hard when the input distribution may depend on the random network. 
Note that in the current work, we give a positive result in a setting where we first fix an input distribution and then draw a random network.
%Our result in the current work implies that when we first fix an input distribution and then draw a random network, learning can be done efficiently. 
%Intuitively, the result in \cite{daniely2020hardness} implies that there is no ``large" class of networks where efficient distribution-free learning is possible, and the current result implies that for every fixed input distribution there exists a ``large" class of networks that can be learned efficiently. 
Finally, learning deep random networks was studied in \citet{das2019learnability,agarwal2021deep}, where the authors showed hardness of learning networks of depth $\omega(\log(d))$ in the SQ model.

\section{Proof of Theorem \ref{thm:main_intro}}

\subsection{Notation}
We recall that for vectors $\bx\in\reals^d$ we use the {\em normalized} Euclidean norm $\|\bx\|=\sqrt{\frac{\sum_{i=1}^dx_i^2}{d}}$ and take the unit sphere $\tilde{\sphere}^{d-1} = \{\bx\in\reals^d : \|\bx\|=1\}$ w.r.t. this norm as our instance space. Inner products will also be normalized: for $\bx,\by\in\reals^d$ we denote $\inner{\bx,\by} = \frac{\sum_{i=1}^dx_iy_i}{d}$.
For $\bx\in \reals^d$ and a closed set $A\subset \reals^d$ we denote $d(\bx,A):=\min_{\bx'\in A}\|\bx-\bx'\|$. 
%For a matrix $W\in\reals^{d_1\times d_2}$ we denote by $W_{i:}$ the $i$'th row of $W$ and by $W_{:j}$ the $j$'th column.
%
Unless otherwise specified, a random scalar is assumed to be a standard normal, a random vector in $\reals^d$ is assumed to be a centered Gaussian vector with covariance matrix $\frac{1}{d}I$, and a random matrix is assumed to be a Xavier matrix.  For $f:\reals\to\reals$, we denote $\|f\|^2=\E_Xf^2(X)$.
We denote the Kronecker delta by $\delta_{ij}$, i.e. $\delta_{ij}=1$ if $i=j$ and $0$ otherwise.

\subsection{Some preliminaries}
We will use the Hermite Polynomials \cite{o2014analysis} which are defined via the following recursion formula.
\begin{equation}\label{eq:Hermite}
    h_{n+1}(x) = \frac{x}{\sqrt{n+1}}h_n(x) - \sqrt{\frac{n}{n+1}}h_{n-1}(x),\;\;h_0(x)=1,\;\;h_1(x)=x
\end{equation}
The Hermite polynomials are the sequence of normalized orthogonal polynomials w.r.t. the standard Gaussian measure. That is, it holds that
\[
\E_Xh_i(X)h_j(X) = \delta_{ij}
\]
More generally, if $(X,Y)$ is a Gaussian vector with covariance matrix $\begin{pmatrix}
    1 & \rho
    \\
    \rho & 1
\end{pmatrix}$ then
\begin{equation}\label{eq:hermite_orthogonality}
\E_{X,Y}h_i(X)h_j(Y) = \delta_{ij}\rho^i
\end{equation}
We will use the fact that
\begin{equation}\label{eq:hermite_derivative}
h'_{n} = \sqrt{n}h_{n-1}
\end{equation}
and that for even $n$
\begin{equation}\label{eq:momnets_of_gaussian}
    \E_X X^n = (n-1)!!
\end{equation}

Let $\sigma = \sum_{i=0}^\infty a_i h_i$ be the representation of the activation function $\sigma$ in the basis of the Hermite polynomials.
We will also use the {\em dual activation} $\hat{\sigma}(\rho) = \sum_{i=0}^\infty a^2_i\rho^i$ as defined in \cite{daniely2016toward}.
We note that $\hat\sigma$ is defined in $[-1,1]$ and satisfies $\hat{\sigma}(1)=\|\sigma\|^2=1$.

\subsection{Some technical lemmas}
\begin{lemma}\label{lem:dial_is_lip}
    If $\sigma$ is $L$-Lipschitz then $\hat\sigma$ is $L^2$-Lipschitz in $[-1,1]$.
\end{lemma}
\begin{proof}
    As shown in \cite{daniely2016toward}, $(\hat{\sigma})' = \widehat{\sigma'}$. Hence, for $\rho\in [-1,1]$,
    \begin{eqnarray*}
    |(\hat{\sigma})'(\rho)| &=&  \left|\widehat{\sigma'}(\rho)\right|
    \\
    &\le& \left\|\sigma'\right\|^2
    \\
    &\le& L^2
    \end{eqnarray*}
\end{proof}

\begin{lemma}\label{lem:hermite_is_lip}
    $|h_n(x)-h_n(x+y)|\le 2^{n}\max(|x|,|x+y|,1)^n|y|$.
\end{lemma}
\begin{proof}
It is not hard to verify by induction on \eqref{eq:Hermite} that
\[
|h_n(x)|\le  2^{n/2}\max(1,|x|^n) 
\]
This implies that for $\xi\in [x,x+y]$
\begin{eqnarray*}
|h_n(x)-h_n(x+y)| &=& |h'_n(\xi)y|
\\
&\stackrel{\eqref{eq:hermite_derivative}}{=}& \sqrt{n}|h_{n-1}(\xi)y|
\\
&\le&  \sqrt{n}2^{n/2}\max(|x|,|x+y|,1)^n|y|
\\
&\le& 2^{n}\max(|x|,|x+y|,1)^n|y|
\end{eqnarray*}

\end{proof}

\subsection{Defining a shadow network}

In order to approximate $\Psi_{\Vec{W}}^i$ via a polynomial, we will use a ``shadow network" that is obtained by replacing the activation $\sigma$ with a polynomial approximation of it. 
We will show that for random networks we can approximate each activation sufficiently well with low-degree Hermite polynomials.
Recall that $\sigma = \sum_{i=0}^\infty a_i h_i$ is the representation of $\sigma$ in the basis of the Hermite polynomials. Define $\sigma_n = \frac{1}{\sqrt{\sum_{i=0}^n a^2_i}}\sum_{i=0}^n a_ih_i$. We have $\epsilon_\sigma(n)  = \sum_{i=n+1}^\infty a_i^2$ and hence
$\sigma_n = \frac{1}{\sqrt{1-\epsilon_\sigma(n)}}\sum_{i=0}^n a_ih_i$.
We next define a shadow network. For $\bx\in\reals^d$ we let $\Psi^{0,n}_{\Vec{W}}(\bx) =\bx$. For $1\le j\le i$ we define recursively 
\[
\Phi^{j,n}_{\Vec{W}}(\bx) = W^j\Psi^{j-1,n}_{\Vec{W}}(\bx),\;\;\;\;
\Psi^{j,n}_{\Vec{W}}(\bx) = \sigma_n\left(\Phi^{j,n}_{\Vec{W}}(\bx)\right)
\]
for $1\le j\le i-1$ and $\Psi^{i,n}_{\Vec{W}}(\bx) = W^i\Psi^{i-1,n}_{\Vec{W}}(\bx)$. We will prove the following theorem, which implies Theorem \ref{thm:main_intro}.

\begin{theorem}\label{thm:main}
Fix $i$ and let $n$ be large enough so that $\epsilon_\sigma(n)\le \frac{1}{2}$.
There is a constant $D = D(n,i,\sigma)$  such that if $d_1,\ldots,d_{i-2} \ge D$ then for any $\bx\in\tilde\sphere^{d-1}$,
\[
\E_{\vec{W}}\left\| \Phi_{\vec{W}}^i(\bx)-\Phi_{\vec{W}}^{i,n}(\bx)  \right\| \le 13 \cdot (L+1)^2 \cdot \left(\epsilon_\sigma(n)\right)^{\frac{1}{2^{i-1}}}
\]   
\end{theorem}

Since $\epsilon_\sigma(n)$ is the error in the approximation of a single activation $\sigma$ with a degree-$n$ polynomial, it is natural to expect that the above bound will depend on $\epsilon_\sigma(n)$.
To see why Theorem \ref{thm:main} (together with Lemma \ref{lem:eps_n_lip} which bounds $\epsilon_\sigma(n$)) implies Theorem \ref{thm:main_intro}, note that $\Phi^{i,n}_{\Vec{W}}(\bx)$ is a polynomial of degree $n^{i-1}$. This implies Theorem \ref{thm:main_intro}, except the requirement that the coefficients of the polynomial are polynomially bounded. To deal with this, define
\[
\tilde \Phi^{i,n}_{\Vec{W}}(\bx) = \begin{cases}
    \Phi^{i,n}_{\Vec{W}}(\bx) & \text{if all entries in $\vec{W}$ are at most $\sum_{j=0}^id_j$} \\ 0 & \text{otherwise}
\end{cases}
\]
As we show next $\lim_{\min(d_1,\ldots,d_{i-1})\to\infty}\E_{\vec{W}}\left\| \Phi^{i,n}_{\Vec{W}}(\bx)-\tilde \Phi^{i,n}_{\Vec{W}}(\bx) \right\| = 0$. Hence, in the theorem we can replace $\Phi^{i,n}$ by $\tilde \Phi^{i,n}$ which has polynomially bounded coefficients. 

\begin{lemma}
 For every $\epsilon$ and $n$ there is a constant $D$ such that if $d_1,\ldots,d_{i-1} \ge D$ then for any $\bx\in\tilde\sphere^{d-1}$,    $\E_{\vec{W}}\left\| \Phi^{i,n}_{\Vec{W}}(\bx)-\tilde \Phi^{i,n}_{\Vec{W}}(\bx) \right\|<\epsilon$.
\end{lemma}
\begin{proof} 
    Let $A$ be the event that there is an entry in $\vec{W}$ that is greater than $\sum_{j=0}^id_j$. We have
    \[
    \E_{\vec{W}}\left\| \Phi^{i,n}_{\Vec{W}}(\bx)-\tilde \Phi^{i,n}_{\Vec{W}}(\bx) \right\| = \E\left[\left\| \Phi^{i,n}_{\Vec{W}}(\bx) \right\|\cdot 1_A\right] \le \sqrt{\E\left\| \Phi^{i,n}_{\Vec{W}}(\bx) \right\|^2}\sqrt{\Pr(A)}
    \]
    Now, it is not hard to verify that $\E\left\| \Phi^{i,n}_{\Vec{W}}(\bx) \right\|^2$ is polynomial in $\sum_{j=0}^id_j$ while $\Pr(A)$ converges to $0$ exponentially fast in $\sum_{j=0}^id_j$. Thus, if $\min(d_1,\ldots,d_{i-1})$ is large enough then $\E_{\vec{W}}\left\| \Phi^{i,n}_{\Vec{W}}(\bx)-\tilde \Phi^{i,n}_{\Vec{W}}(\bx) \right\|<\epsilon$.
\end{proof}

\begin{lemma}
    $\tilde \Phi^{i,n}_{\vec{W}}$ computes a polynomial whose sum of coefficients is at most $(2\bar d)^{4n^{i-1}}$.
\end{lemma}
\begin{proof}
We assume that $\tilde \Phi^{i,n}_{\vec{W}} = \Phi^{i,n}_{\vec{W}} $, as otherwise $\tilde \Phi^{i,n}_{\vec{W}} \equiv 0$, in which case the lemma is clear.
Write $\sigma_n(x) = \sum_{k=0}^n b_{k}x^k$ and $h_j(x) = \sum_{k=0}^j c_{j,k}x^k$. Via induction on \eqref{eq:Hermite}, we have $|c_{j,k}|\le 2^{\frac{j}{2}}$. Hence, 
\begin{eqnarray*}
|b_k| &\le& \frac{1}{\sqrt{\sum_{j=0}^n a_j^2}}\sum_{j=0}^n |a_j||c_{j,k}|
\\
&\le& \frac{1}{\sqrt{\sum_{j=0}^n a_j^2}}\sum_{j=0}^n |a_j|2^{\frac{j}{2}}
\\
&\le & \frac{1}{\sqrt{\sum_{j=0}^n a_j^2}} \sqrt{\sum_{j=0}^n a_j^2}\sqrt{\sum_{j=0}^n 2^j}
\\
&\le & 2^{\frac{n+1}{2}}
\end{eqnarray*}

Now, let $M_{j}$ be the maximal sum of coefficients of any polynomial computed by an output neuron of $\Psi^{j,n}_{\vec{W}}$.
We next show by induction that  $M_j\le (2\bar d)^{2\sum_{k=1}^j n^k}$. 
This will conclude the proof as it will imply that the sum of the coefficients of the polynomial computed by $\Phi^{i,n}_{\vec{W}}$ is at most $(2\bar d)^2 M_{i-1} \le (2\bar d)^{2\sum_{k=0}^{i-1} n^k} \le (2\bar d)^{4n^{i-1}}$.
For $j=0$ we have  $M_0=1$. For $j\ge 1$ we we have
\[
M_{j} \le \sum_{k=0}^n |b_k|\left((\bar d)^2 M_{j-1}\right)^{k} \le 2^{\frac{n+1}{2}}\cdot 2\cdot  \left((\bar d)^2 M_{j-1}\right)^{n}\le \left((2 \bar d)^2 M_{j-1}\right)^{n} 
\]
By the induction hypothesis we have
\[
M_j \le (2 \bar d)^{2n + 2n\sum_{k=1}^{j-1} n^k} = (2 \bar d)^{2\sum_{k=1}^{j} n^k}
\]
\end{proof}

\subsection{Proof of Theorem \ref{thm:main} for depth-$2$ networks}\label{sec:proof_depth_two}
We will first prove Theorem \ref{thm:main} for depth-$2$ networks (i.e. for $i=2$). 
We will prove Lemma \ref{lem:single_layer_normalized} below which implies that for every $\epsilon$ there is 
%a constant 
$n$ such that for any $\bx\in\tilde\sphere^{d-1}$, $\E_{\vec{W}}\left\| \Psi^{1,n}_{\Vec{W}}(\bx)-\Psi^{1}_{\Vec{W}}(\bx)  \right\| \le \epsilon$. We will then prove Lemma \ref{lem:adding_linear_layer}, that together with Lemma \ref{lem:single_layer_normalized} will show that $\E_{\vec{W}}\left\| \Phi^{2,n}_{\Vec{W}}(\bx)-\Phi^{2}_{\Vec{W}}(\bx)  \right\| \le \epsilon$, thus proving Theorem \ref{thm:main} for $i=2$. We will start however with the following lemma that will be useful throughout.

\begin{lemma}\label{lem:vec_to_scalar_expected_val}
    Fix $f,g:\reals\to\reals$, $\bx,\by\in\reals^{d_1}$ and a Xavier matrix $W\in\reals^{d_2\times d_1}$. Let $(X,Y)$ be a centered Gaussian vector with covariance matrix $\begin{pmatrix}\|\bx\|^2 & \inner{\bx,\by} \\ \inner{\bx,\by} & \|\by\|^2  \end{pmatrix}$. Then
    \[
    \E_W\left\|f(W\bx) - g(W\by)\right\| \leq \sqrt{\E_W\left\|f(W\bx) - g(W\by)\right\|^2}  = \sqrt{\E_{X,Y}(f(X) - g(Y))^2}
    \]    
\end{lemma}
\begin{proof}
We have
\begin{eqnarray*}
\E_W\left\|f(W\bx) - g(W\by)\right\| &\stackrel{\text{Jensen Inequality}}{\le}& \sqrt{\E_W\left\|f(W\bx) - g(W\by)\right\|^2}
\\
&=& \sqrt{\frac{1}{d_2}\sum_{j=1}^{d_2}\E_W(f((W\bx)_j) - g((W\by)_j)^2}
\end{eqnarray*}
Now, the lemma follows from the fact that $\{\left((W\bx)_j,(W\by)_j\right)\}_{j=1}^{d_2}$ are independent centered Gaussian vectors with covariance matrix $\begin{pmatrix}\|\bx\|^2 & \inner{\bx,\by} \\ \inner{\bx,\by} & \|\by\|^2  \end{pmatrix}$.
\end{proof}

\begin{lemma}\label{lem:single_layer_normalized}
Fix $\bx \in \tilde\sphere^{d_1-1}$. Let $W\in \reals^{d_2\times d_1}$ be a Xavier matrix. Then
\[
\E_W \left\|\sigma(W\bx) - \sigma_n(W\bx)\right\| \le \sqrt{2\epsilon_\sigma(n)}
\]
\end{lemma}
\begin{proof}
By Lemma \ref{lem:vec_to_scalar_expected_val} we have
%Let $\tilde \sigma_n(x) = \sum_{i=0}^na_ih_i(x)$.
\begin{eqnarray*}
\E_W\left\|\sigma(W\bx) - \sigma_n(W\bx)\right\| &\le & \sqrt{\E_W\left\|\sigma(W\bx) - \sigma_n(W\bx)\right\|^2}
\\
&=& \sqrt{\E_{X}(\sigma(X) - \sigma_n(X))^2} 
\\
&=& \sqrt{\sum_{i=0}^n\left(1-\frac{1}{\sqrt{1-\epsilon_\sigma(n)}}\right)^2a^2_i + \sum_{i=n+1}^\infty a_i^2} 
\\
&=& \sqrt{(1-\epsilon_\sigma(n))\left(1-\frac{1}{\sqrt{1-\epsilon_\sigma(n)}}\right)^2 + \epsilon_\sigma(n)} 
\\
&=& \sqrt{(1-\epsilon_\sigma(n))\left(\frac{\sqrt{1-\epsilon_\sigma(n)}-1}{\sqrt{1-\epsilon_\sigma(n)}}\right)^2 + \epsilon_\sigma(n)} 
\\
&=& \sqrt{2 -\epsilon_\sigma(n) - 2\sqrt{1-\epsilon_\sigma(n)}   + \epsilon_\sigma(n)} 
\\
&=& \sqrt{2(1-\sqrt{1-\epsilon_\sigma(n)})} 
\\
&\le& \sqrt{2(1-\sqrt{1-\epsilon_\sigma(n)})(1+\sqrt{1-\epsilon_\sigma(n)})} 
\\
&=& \sqrt{2\epsilon_\sigma(n)}   
\end{eqnarray*}
\end{proof}

Lemma \ref{lem:single_layer_normalized} implies that $\E_{\vec{W}}\left\| \Psi^{1,n}_{\Vec{W}}(\bx)-\Psi^{1}_{\Vec{W}}(\bx)  \right\| \le \sqrt{2\epsilon_\sigma(n)}$. Thus, given $\epsilon>0$, for sufficiently large $n$, $\E_{\vec{W}}\left\| \Psi^{1,n}_{\Vec{W}}(\bx)-\Psi^{1}_{\Vec{W}}(\bx)  \right\| \le \epsilon$. The following lemma therefore implies that $\E_{\vec{W}}\left\| \Phi^{2,n}_{\Vec{W}}(\bx)-\Phi^{2}_{\Vec{W}}(\bx)  \right\| \le \sqrt{2\epsilon_\sigma(n)}$ and thus implies Theorem \ref{thm:main} for depth two networks.

\begin{lemma}\label{lem:adding_linear_layer}
For any $\bx \in \tilde\sphere^{d-1}$
\[
\E_{W^i}\left\| \Phi^{i,n}_{\Vec{W}}(\bx)-\Phi^{i}_{\Vec{W}}(\bx)  \right\| \le 
\left\| \Psi^{i-1,n}_{\Vec{W}}(\bx)-\Psi^{i-1}_{\Vec{W}}(\bx) \right\|
\]
\end{lemma}
\begin{proof}
By Lemma \ref{lem:vec_to_scalar_expected_val} we have
\begin{eqnarray*}
\E_{W^i}\left\| \Phi^{i,n}_{\Vec{W}}(\bx)-\Phi^{i}_{\Vec{W}}(\bx)  \right\| &=& \E_{W^i}\left\| W^i\left(\Psi^{i-1,n}_{\Vec{W}}(\bx)-\Psi^{i-1}_{\Vec{W}}(\bx)\right)  \right\|
\\
&\le & \sqrt{\E_{X\sim\Ncal\left(0,\left\| \Psi^{i-1,n}_{\Vec{W}}(\bx)-\Psi^{i-1}_{\Vec{W}}(\bx) \right\|^2\right)}X^2}
\\
&= & \left\| \Psi^{i-1,n}_{\Vec{W}}(\bx)-\Psi^{i-1}_{\Vec{W}}(\bx) \right\|
\end{eqnarray*}
\end{proof}

\subsection{Proof of Theorem \ref{thm:main} for general networks}
For $\bx\in\tilde\reals^{d_{i-1}}$ we denote $\Psi_{W^i}(\bx) = \sigma(W^i\bx)$ and $\Psi^n_{W^i}(\bx) = \sigma_n(W^i\bx)$.
Lemma \ref{lem:single_layer_normalized} can be roughly phrased as
\[
(\bx=\bx')\text{ and }(\|\bx\|=1) \Rightarrow \Psi_{W^i}(\bx)\approx \Psi^n_{W^i}(\bx')
\]
In order to prove Theorem \ref{thm:main} for general networks we will extend it by replacing the strict equality conditions with softer ones. That is, we will show that
\begin{equation}\label{eq:adding_layer_sketch}
(\bx\approx\bx')\text{ and }(\|\bx\|\approx1) \text{ and }(\|\bx'\|\approx1)\Rightarrow \Psi_{W^i}(\bx)\approx \Psi^n_{W^i}(\bx')
\end{equation}
This will be enough to prove Theorem \ref{thm:main} for general networks. Indeed, the conditions $\|\bx\|\approx1$ and $\|\bx'\|\approx 1$ are valid w.h.p. via a simple probabilistic argument. Thus, \eqref{eq:adding_layer_sketch} implies that
\begin{equation}\label{eq:adding_layer_sketch_without_sphere}
\bx\approx \bx'\Rightarrow \Psi_{W^i}(\bx)\approx \Psi^n_{W^i}(\bx')
\end{equation}
Now,  for $\bx\in\tilde\sphere^{d-1}$ \eqref{eq:adding_layer_sketch_without_sphere} implies that 
$\Psi_{W^1}(\bx)\approx \Psi^n_{W^1}(\bx')$. Using \eqref{eq:adding_layer_sketch_without_sphere} again we get that $\Psi_{W^2}\circ\Psi_{W^1}(\bx)\approx \Psi^n_{W^2}\circ\Psi^n_{W^1}(\bx')$. Using it $i-3$ more times we we get that $\Psi_{W^{i-1}}\circ\cdots\circ\Psi_{W^1}(\bx)\approx \Psi^n_{W^{i-1}}\circ\Psi^n_{W^1}(\bx')$, or in other words that $\Psi^{i-1}_{\vec{W}}(\bx)\approx\Psi^{i-1,n}_{\vec{W}}(\bx)$. As we will show $``\approx"$ stands for a sufficiently strong approximation, which guarantees that $\E_{\vec{W}}\|\Psi^{i-1}_{\vec{W}}(\bx)-\Psi^{i-1,n}_{\vec{W}}(\bx)\|\le \epsilon$, and hence Lemma \ref{lem:adding_linear_layer} implies Theorem \ref{thm:main}.

To prove \eqref{eq:adding_layer_sketch} we first prove Lemma \ref{lem:adding_layer_normalized} which softens the requirement that $\bx=\bx'$. That is, it shows that
\[
(\bx\approx\bx')\text{ and }(\|\bx\|=\|\bx'\|= 1) \Rightarrow \Psi_{W^i}(\bx)\approx \Psi^n_{W^i}(\bx')
\]
The second condition which requires that $\|\bx\|=\|\bx'\|= 1$ is softened via Lemmas \ref{lem:dist_preserving_proj} and \ref{lem:projecting_dont_hurt}. Lemma \ref{lem:main_lemma} then wraps the two softenings together, and shows that \eqref{eq:adding_layer_sketch} is valid. Finally, in section \ref{sec:conclude_proof} we use Lemma \ref{lem:main_lemma} to prove Theorem \ref{thm:main}.

\begin{lemma}\label{lem:adding_layer_normalized}
Fix $\bx,\bx+\bv\in \tilde\sphere^{d_1-1}$ with $\|\bv\| \le \epsilon $. Let $W\in \reals^{d_2\times d_1}$ be a Xavier matrix. Then
\[
\E_W \left\|\sigma(W\bx) - \sigma_n(W(\bx+\bv))\right\| \le \sqrt{2\epsilon_\sigma(n)} + \sqrt{\frac{2L^2}{1-\epsilon_\sigma(n)}\epsilon}
\]
\end{lemma}
\begin{proof}
We have
\[
\left\|\sigma(W\bx) - \sigma_n(W(\bx+\bv))\right\| \le \left\|\sigma(W\bx) - \sigma_n(W\bx)\right\| + \left\| \sigma_n(W\bx) - \sigma_n(W(\bx+\bv))\right\|
\]
By Lemma \ref{lem:single_layer_normalized} we have $\E_W\left\|\sigma(W\bx) - \sigma_n(W\bx)\right\|\le \sqrt{2\epsilon_\sigma(n)}$. It remains to bound $\E_W\left\|\sigma_n(W\bx) - \sigma_n(W(\bx+\bv))\right\|$. By Lemma \ref{lem:vec_to_scalar_expected_val} We have
\[
\E_W\left\|\sigma_n(W\bx) - \sigma_n(W(\bx+\bv))\right\| \le \sqrt{\E_{X,Y}(\sigma_n(X) - \sigma_n(Y))^2}
\]
where $(X,Y)$ is a centered Gaussian vector with correlation matrix $\begin{pmatrix}1 & \rho\\   \rho & 1 \end{pmatrix}$ for $\rho = \inner{\bx,\bx + \bv} \ge 1-\epsilon$.
Finally, we have 
\begin{eqnarray*}
\E_{X,Y}(\sigma_n(X)-\sigma_n(Y))^2 &=&  \frac{1}{1-\epsilon_\sigma(n)}\E_{X,Y}\left(\sum_{i=0}^na_i(h_i(X)-h_i(Y))\right)^2
\\
 &=& \frac{1}{1-\epsilon_\sigma(n)}\sum_{i=0}^n\sum_{j=0}^n a_ia_j\E_{X,Y}(h_i(X)-h_i(Y))(h_j(X)-h_j(Y))
\\
&\stackrel{\eqref{eq:hermite_orthogonality}}{=}& \frac{1}{1-\epsilon_\sigma(n)}\sum_{i=0}^n a^2_i(2-2\rho^i)
\\
&\le & \frac{2}{1-\epsilon_\sigma(n)}\left(\hat\sigma(1) - \hat\sigma(\rho)\right)
\\
&\stackrel{\text{Lemma \ref{lem:dial_is_lip}}}{\le} & \frac{2L^2}{1-\epsilon_\sigma(n)}\epsilon
\end{eqnarray*} 
\end{proof}

We next prove a lemma that allows us to ``almost jointly project" a pair of points $\bx_1,\bx_2\in\reals^d$ on a closed set $A\subset \reals^d$, without expanding the distance too much. 

\begin{lemma}\label{lem:dist_preserving_proj}
Let $A\subset\reals^d$ a closed set and fix $\bx_1,\bx_2 \in \reals^d$. There are $\tilde \bx_1,\tilde \bx_2 \in A$ such that
\[
\|\bx_1 - \tilde \bx_1\| \le 2d(\bx_1,A),\;\;\|\bx_2 - \tilde \bx_2\| \le 2d(\bx_2,A)\;\;\text{ and }\;\;\|\tilde \bx_1 - \tilde \bx_2\|\le 3\|\bx_1 - \bx_2\|
\]
\end{lemma}

\begin{proof}
Let $P_A:\reals^d\to A$ a function such that for any $\bx\in\reals^d$, $\|P_A(\bx) - \bx\| = d(\bx,A)$.
Assume w.l.o.g. that $\left\|\bx_1 - P_A(\bx_1) \right\| \le \left\|\bx_2 - P_A(\bx_2) \right\|$.

\subsubsection*{Case I: $\left\|\bx_2 - P_A(\bx_2) \right\| \le \| \bx_1 - \bx_2 \|$}
Simply define $\tilde \bx_i =P_A(\bx_i)$. We have
\[
\|\bx_1 - \tilde \bx_1\| = \left\|\bx_1 -  P_A(\bx_1) \right\|,\;\;\;\|\bx_2 - \tilde \bx_2\| = \left\|\bx_2 - P_A(\bx_2)\right\|
\]
and
\[
\|\tilde \bx_1 - \tilde \bx_2\| \le
\left\| P_A(\bx_1)-\bx_1 \right\| + \|\bx_1-\bx_2\| + \left\|\bx_2 - P_A(\bx_2)\right\|
\le 3\| \bx_1 -  \bx_2\| 
\]
\subsubsection*{Case II: $ \| \bx_1 - \bx_2 \| \le \left\|\bx_2 - P_A(\bx_2) \right\| $}
Define $\tilde \bx_1 = \tilde \bx_2 = P_A(\bx_1)$. We have
\[
\|\bx_1 - \tilde \bx_1\| = \left\|\bx_1 -  P_A(\bx_1) \right\|,\;\;\|\tilde \bx_1 - \tilde \bx_2\| \le 0\| \bx_1 -  \bx_2\| 
\]
and
\[
\|\bx_2 - \tilde \bx_2\| \le \left\|\bx_2 -  \bx_1 \right\| +  \left\|\bx_1 -  P_A(\bx_1) \right\|\le 2\left\|\bx_2 -  P_A(\bx_2) \right\|
\]
\end{proof}

\begin{lemma}\label{lem:projecting_dont_hurt}
Let $\bx,\bx+\bv\in \reals^{d_1}$ be vectors such that $\|\bx\| = 1$ and $\|\bv\| \le \epsilon \le 1$.
Let $W\in \reals^{d_2\times d_1}$ be a Xavier matrix. Then 
\[
\E_W \left\|\sigma(W\bx) - \sigma(W(\bx+\bv))\right\| \le 
L\epsilon
\]
and
\[
\E_W \left\|\sigma_n(W\bx) - \sigma_n(W(\bx+\bv))\right\| \le 
2^{2n+1}\left(9(4n-1)!!\right)^{1/4} \epsilon =: \lambda(n)\epsilon
\]
\end{lemma}
\begin{proof}
Fix a centered Gaussain vector $(X,Y)$ with covariance matrix $\begin{pmatrix}
    1 & \inner{\bx + \bv,\bx} \\ \inner{\bx + \bv,\bx} & \|\bx + \bv\|^2
\end{pmatrix}$. Let $Z = Y - X$. Note that  $\var(Z) \le \epsilon^2$. By Lemma \ref{lem:vec_to_scalar_expected_val} we have
\[
\E_W \left\|\sigma(W\bx) - \sigma(W(\bx+\bv))\right\| \le \sqrt{\E(\sigma(X)-\sigma(X+Z))^2}\le  \sqrt{L^2\E Z^2} \le L \epsilon
\]
For the second part, we have by Lemma \ref{lem:hermite_is_lip} that
\begin{align*}
|\sigma_n(x) - \sigma_n(x+y)|
&\le \sum_{i=0}^n \frac{|a_i|}{\sqrt{1-\epsilon_\sigma(n)}} |h_i(x) - h_i(x+y)| 
\\
&\le |y|\sum_{i=0}^n 2^{i}\max(|x|^i,|x+y|^i,1)
\\
&\le |y|2^{n+1}\max(|x|^n,|x+y|^n,1)
\end{align*}
Hence,
\begin{eqnarray*}
\E(\sigma_n(X)-\sigma_n(X+Z))^2 &\le& 2^{2n+2}\E Z^2\max(|X|^n,|X+Z|^n,1)^{2}
\\
&\le& 2^{2n+2}\sqrt{\E Z^4}   \sqrt{\E\max(|X|^{4n},|X+Z|^{4n},1)}    
\\
&\le& 2^{2n+2}\sqrt{\E Z^4}   \sqrt{\E\left[|X|^{4n} + |X+Z|^{4n} + 1\right]}    
\\
&\stackrel{\eqref{eq:momnets_of_gaussian}}{=}& 2^{2n+2}\sqrt{3\|\bv\|^4}   \sqrt{1 + (4n-1)!!(\|\bx+\bv\|^{4n} +  \|\bx\|^{4n})}    
\\
&\le & 2^{2n+2}\sqrt{3\epsilon^4}   \sqrt{3(1+\epsilon)^{4n}(4n-1)!!}
\\
&\le & 2^{2n+2}\sqrt{3\epsilon^4}   \sqrt{3\cdot 2^{4n}\cdot (4n-1)!!}
\end{eqnarray*}
Lemma \ref{lem:vec_to_scalar_expected_val} now implies that
\begin{eqnarray*}
\E_W \left\|\sigma_n(W\bx) - \sigma_n(W(\bx+\bv))\right\| &\le&
2^{2n+1}\left(9(4n-1)!!\right)^{1/4} \epsilon
% \\
% &\le& L\sqrt{d_2 6 (1+\epsilon)^{2n}
% \sqrt{(2n)!2^{2n} }}\epsilon
% \\
% &\le& L\sqrt{d_2 6 (1+\epsilon)^{2n}
% \sqrt{(2n)^{2n+1}(2/e)^{2n} }}\epsilon
\end{eqnarray*}
\end{proof}

\begin{lemma}\label{lem:main_lemma}
Let $\bx,\bx+\bv\in \reals^{d_1}$ be vectors such that $\|\bv\| \le \epsilon $, $|\|\bx\| - 1|\le \delta\le 1/2$ and $|\|\bx +\bv\| - 1|\le \delta$. 
Let $W\in \reals^{d_2\times d_1}$ be a Xavier matrix. Then
\[
\E_W \left\|\sigma(W\bx) - \sigma_n(W(\bx+\bv))\right\| \le 
2L\delta +\sqrt{2\epsilon_\sigma(n)} + \sqrt{\frac{6L^2}{1-\epsilon_\sigma(n)}\epsilon} + 2\lambda(n)\delta
\]
\end{lemma}
\begin{proof}
By Lemma \ref{lem:dist_preserving_proj} there are vectors $\bx',\bv'$ such that $\|\bx'\| = \|\bx' + \bv'\| = 1$ and
\[
\|\bx - \bx'\| \le 2\delta,\; \|\bx + \bv - \bx' - \bv'\| \le 2\delta,\;\text{ and }\|\bv'\| \le 3\|\bv\|
\]
Now, we have, by Lemmas \ref{lem:adding_layer_normalized} and \ref{lem:projecting_dont_hurt},
\begin{eqnarray*}
    \E_W \left\|\sigma(W\bx) - \sigma_n(W(\bx+\bv))\right\| &\le& \E_W \left\|\sigma(W\bx) - \sigma(W\bx')\right\| 
    \\
    && + \E_W \left\|\sigma(W\bx') - \sigma_n(W(\bx'+\bv'))\right\|
    \\
    && + \E_W \left\|\sigma_n(W(\bx'+\bv')) - \sigma_n(W(\bx+\bv))\right\|
    \\
    &\le & 2L\delta +\sqrt{2\epsilon_\sigma(n)} + \sqrt{\frac{6L^2}{1-\epsilon_\sigma(n)}\epsilon} + 2\lambda(n)\delta
\end{eqnarray*}
\end{proof}

\subsubsection{Concluding the proof of Theorem \ref{thm:main}}\label{sec:conclude_proof}
Define
\[
\Psi_{\vec{W}}^i(\bx,\delta) = \begin{cases} 0 & |1 - \|\Psi_{\vec{W}}^j(\bx)\||>\delta\text{ or }|1 - \|\Psi_{\vec{W}}^{j,n}(\bx)\||>\delta\text{ for some }j<i
\\
\Psi_{\vec{W}}^i(\bx) & \text{otherwise} \end{cases}
\]
and
\[
\Psi_{\vec{W}}^{i,n}(\bx,\delta) = \begin{cases} 0 & |1 - \|\Psi_{\vec{W}}^j(\bx)\||>\delta\text{ or }|1 - \|\Psi_{\vec{W}}^{j,n}(\bx)\||>\delta\text{ for some }j<i
\\
\Psi_{\vec{W}}^{i,n}(\bx) & \text{otherwise} \end{cases}
\]
We have
\begin{align*}
\E_{\vec{W}}\left\| \Psi_{\vec{W}}^i(\bx)-\Psi_{\vec{W}}^{i,n}(\bx)  \right\| 
&\le
\E_{\vec{W}}\left\| \Psi_{\vec{W}}^i(\bx)-\Psi_{\vec{W}}^i(\bx,\delta)  \right\| +
\E_{\vec{W}}\left\| \Psi_{\vec{W}}^i(\bx,\delta)-\Psi_{\vec{W}}^{i,n}(\bx,\delta) \right\| \\
&\;\;\;\;+
\E_{\vec{W}}\left\| \Psi_{\vec{W}}^{i,n}(\bx,\delta)-\Psi_{\vec{W}}^{i,n}(\bx)  \right\|
\end{align*}
Theorem \ref{thm:main} now follows from Lemmas \ref{lem:contraction} and \ref{lem:remove_contraction} below, together with Lemma \ref{lem:adding_linear_layer}.

\begin{lemma}\label{lem:contraction}
Let $n$ be large enough so that $\epsilon_\sigma(n)\le \frac{1}{2}$ and let $\delta<\frac{\sqrt{\epsilon_\sigma(n)}}{2L+2\lambda(n)}$. Then,
\[
\E_{\vec{W}}\left\| \Psi_{\vec{W}}^i(\bx,\delta)-\Psi_{\vec{W}}^{i,n}(\bx,\delta)  \right\| \le 12 \cdot (L+1)^2 \cdot \left(\epsilon_\sigma(n)\right)^{2^{-i}}
\]   
\end{lemma}
\begin{proof}
We will prove the result by induction on $i$. The case $i=0$ is clear as $\Psi_{\vec{W}}^0(\bx,\delta)=\Psi_{\vec{W}}^{0,n}(\bx,\delta)$. Fix $i>0$.
For every $\delta<\frac{1}{2}$ and $n$ we have by Lemma \ref{lem:main_lemma}
\small
\begin{eqnarray*}
\E_{W^i}\left\| \Psi_{\vec{W}}^i(\bx,\delta)-\Psi_{\vec{W}}^{i,n}(\bx,\delta)  \right\| 
&\le & 2L\delta +\sqrt{2\epsilon_\sigma(n)} + \sqrt{\frac{6L^2}{1-\epsilon_\sigma(n)}\left\| \Psi_{\vec{W}}^{i-1}(\bx,\delta)-\Psi_{\vec{W}}^{i-1,n}(\bx,\delta)  \right\| } + 2\lambda(n)\delta
\end{eqnarray*}
\normalsize
Taking expectation over $W^1,\ldots,W^{i-1}$ we get
\small
\begin{eqnarray*}
\E_{\vec{W}}\left\| \Psi_{\vec{W}}^i(\bx,\delta)-\Psi_{\vec{W}}^{i,n}(\bx,\delta)  \right\| &\le& 2L\delta +\sqrt{2\epsilon_\sigma(n)} + \E_{\vec{W}}\sqrt{\frac{6L^2}{1-\epsilon_\sigma(n)}\left\| \Psi_{\vec{W}}^{i-1}(\bx,\delta)-\Psi_{\vec{W}}^{i-1,n}(\bx,\delta)  \right\| } + 2\lambda(n)\delta
\\
&\stackrel{\text{Jensen}}{\le} & 2L\delta +\sqrt{2\epsilon_\sigma(n)} + \sqrt{\frac{6L^2}{1-\epsilon_\sigma(n)}\E_{\vec{W}}\left\| \Psi_{\vec{W}}^{i-1}(\bx,\delta)-\Psi_{\vec{W}}^{i-1,n}(\bx,\delta)  \right\| } + 2\lambda(n)\delta
\\
&\stackrel{\delta < \frac{\sqrt{\epsilon_\sigma(n)}}{2L+2\lambda(n)}}{\le}& 4\sqrt{\epsilon_\sigma(n)} + \sqrt{\frac{6L^2}{1-\epsilon_\sigma(n)}\E_{\vec{W}}\left\| \Psi_{\vec{W}}^{i-1}(\bx,\delta)-\Psi_{\vec{W}}^{i-1,n}(\bx,\delta)  \right\| }
\\
&\stackrel{\epsilon_\sigma(n)\le \frac{1}{2}}{\le}& 4\sqrt{\epsilon_\sigma(n)} + L\sqrt{12\E_{\vec{W}}\left\| \Psi_{\vec{W}}^{i-1}(\bx,\delta)-\Psi_{\vec{W}}^{i-1,n}(\bx,\delta)  \right\| }
\\
&\stackrel{\text{Induction}}{\le}& 4\sqrt{\epsilon_\sigma(n)} + L\sqrt{12 \cdot 12 \cdot (L+1)^2 \cdot \left(\epsilon_\sigma(n)\right)^{2^{-i+1}}}
\\
&\le& (L+1)\sqrt{12 \cdot 12 \cdot (L+1)^2 \cdot \left(\epsilon_\sigma(n)\right)^{2^{-i+1}}}
\\
&=&12 \cdot (L+1)^2 \cdot \left(\epsilon_\sigma(n)\right)^{2^{-i}}
\end{eqnarray*}
\normalsize
\end{proof}

\begin{lemma}\label{lem:remove_contraction}
Fix $i,n, \delta$ and $\epsilon>0$. There is a constant $D$ such that if $d_1,\ldots,d_{i-1}\ge D$ then
\[
\E_{\vec{W}}\left\| \Psi_{\vec{W}}^i(\bx)-\Psi_{\vec{W}}^i(\bx,\delta)  \right\| +
\E_{\vec{W}}\left\| \Psi_{\vec{W}}^{i,n}(\bx,\delta)-\Psi_{\vec{W}}^{i,n}(\bx)  \right\| \le \epsilon
\]   
\end{lemma}
\begin{proof}
Let $B_{i,\delta}$ be the event that for some $j<i$, $|1 - \|\Psi_{\vec{W}}^j(\bx)\||>\delta$ or $|1 - \|\Psi_{\vec{W}}^{j,n}(\bx)\||>\delta$. We have
\[
\E_{\vec{W}}\left\| \Psi_{\vec{W}}^i(\bx)-\Psi_{\vec{W}}^i(\bx,\delta)  \right\| = \E_{\vec{W}}\left[ \left\| \Psi_{\vec{W}}^i(\bx)  \right\| 1_{B_{i,\delta}} \right]  \le  \sqrt{\E_{\vec{W}}\left[ \left\| \Psi_{\vec{W}}^i(\bx)  \right\|^2  \right]}\sqrt{\Pr(B_{i,\delta})}
\]
Similarly,
\[
\E_{\vec{W}}\left\| \Psi_{\vec{W}}^{i,n}(\bx)-\Psi_{\vec{W}}^{i,n}(\bx,\delta)  \right\| \le  \sqrt{\E_{\vec{W}}\left[ \left\| \Psi_{\vec{W}}^{i,n}(\bx)  \right\|^2  \right]}\sqrt{\Pr(B_{i,\delta})}
\]
the lemma now follows from the following two claims.
\begin{claim}
$\E_{\vec{W}} \left\| \Psi_{\vec{W}}^i(\bx)  \right\|^2$ and $\E_{\vec{W}} \left\| \Psi_{\vec{W}}^{i,n}(\bx)  \right\|^2$ are bounded by a constant 
%that depends only on $i$ and $L$ (but not on $d_0,\ldots,d_{i}$).
(independent of $d_0,\ldots,d_{i}$).
\end{claim}
\begin{proof}
We have
\[
\E_{W^i} \left\| \Psi_{\vec{W}}^i(\bx)  \right\|^2 = \E_{\bw}\sigma^2\left(\bw^\top\Psi_{\vec{W}}^{i-1}(\bx) \right) \le 2\sigma^2(0) + 2L^2 \E_{\bw}\left(\bw^\top  \Psi_{\vec{W}}^{i-1}(\bx)\right)^2 =2\sigma^2(0) + 2L^2 \|\Psi_{\vec{W}}^{i-1}(\bx)\|^2
\]
By induction on $i$, this implies that $\E_{\vec{W}}\left\| \Psi_{\vec{W}}^i(\bx)  \right\|^2$ is bounded by a constant that depends only on $i$ and $L$ (but not on $d_1,\ldots,d_i$).
For $\E_{\vec{W}} \left\| \Psi_{\vec{W}}^{i,n}(\bx)  \right\|^2$ we have
\[
\E_{W^i} \left\| \Psi_{\vec{W}}^{i,n}(\bx)  \right\|^2= \E_{\bw}\sigma_n^2\left(\bw^\top  \Psi_{\vec{W}}^{i-1,n}(\bx)\right)
\]
Hence, $\E_{W^i} \left\| \Psi_{\vec{W}}^{i,n}(\bx)  \right\|^2 $ is an even polynomial in $\left\| \Psi_{\vec{W}}^{i-1,n}(\bx)  \right\|$ of degree $\le 2n$. The polynomial depends only on $\sigma_n$.
It therefore enough to show that for any $i$ and $k$, $\E_{\vec{W}}\left\| \Psi_{\vec{W}}^{i,n}(\bx)  \right\|^{2k}$ is bounded, by a bound that is independent of $d_0,\ldots, d_i$. We will show that via induction on $i$. For $i=0$ this is trivial as $\left\| \Psi_{\vec{W}}^{0,n}(\bx)  \right\|^{2k}\equiv 1$. Fix $i\ge 1$. We have
\begin{eqnarray*}
\E_{W^i} \left\| \Psi_{\vec{W}}^{i,n}(\bx)  \right\|^{2k}  &=& \E_{W^i} \left(\frac{\sum_{j=1}^{d_i} \sigma^2_{n}\left(\left(W^i 
 \Psi_{\vec{W}}^{i-1,n}(\bx)\right)_j\right) }{d_i}\right)^k
\\
&\stackrel{\text{Jensen inequality}}{\le}& \frac{1}{d_i}\E_{W^i} \sum_{j=1}^{d_i} \sigma^{2k}_{n}\left(\left(W^i 
 \Psi_{\vec{W}}^{i-1,n}(\bx)\right)_j\right) 
\\
&=& \E_{\bw}  \sigma^{2k}_{n}\left(\bw^\top \Psi_{\vec{W}}^{i-1,n}(\bx)\right) 
\end{eqnarray*}
The last expression is an even polynomial in $\|\Psi_{\vec{W}}^{i-1,n}(\bx)\|$. The polynomial depends only on $2k$ and $n$.
By the induction hypothesis we conclude that $\E_{\vec{W}}\left\| \Psi_{\vec{W}}^{i,n}(\bx)  \right\|^{2k}$ is bounded by a bound that is independent from $d_0,\ldots,d_i$.
\end{proof}

\begin{claim}
    For every $\delta,\epsilon'$, $i$ and $n$, there is a constant $D$ such that if $d_1,\ldots,d_{i-1}\ge D$ then  $\Pr(B_{i,\delta}) < \epsilon'$.
\end{claim}
\begin{proof}
    We will prove the lemma by induction on $i$. For $i=1$ this is immediate as $\Pr(B_{i,\delta}) = 0$. Fix $i\ge 2$. Let $\delta'$ be small enough so that if $\left|\|\bx\| - 1\right|\le \delta'$ then
    \[
    \left|\E_{\bw}\sigma^2(\bw^\top\bx) - 1\right| < \frac{\delta}{4}\text{ and }\left|\E_{\bw}\sigma_n^2(\bw^\top\bx) - 1\right| < \frac{\delta}{4}
    \]
    and
    \[
    \left|\E_{\bw}\sigma^4(\bw^\top\bx) - \E_{X}\sigma^4(X)\right| < 1\text{ and }\left|\E_{\bw}\sigma_n^4(\bw^\top\bx) - \E_{X}\sigma_n^4(X)\right| < 1
    \]
    we have
    \[
    \Pr(B_{i,\delta}) \le \Pr(B_{i,\delta} | B^c_{i-1,\delta'} ) + \Pr(B_{i-1,\delta'})
    \]
    By Chebyshev inequality, $\Pr(B_{i,\delta} | B^c_{i-1,\delta'} ) < \frac{\epsilon'}{2}$ for sufficiently large $d_{i-1}$. By the induction hypothesis, $\Pr(B_{i-1,\delta'}) < \frac{\epsilon'}{2}$ for sufficiently large $d_1,\ldots, d_{i-2}$
\end{proof}
\end{proof}

\subsection{Bounds on $\epsilon_\sigma(n)$}\label{sec:bound_on_eps}
By \eqref{eq:hermite_derivative} if $\sigma$ is differentiable $k$ times then we have $\sigma^{(k)} = \sum_{i=k}^\infty \sqrt{\frac{i!}{(i-k)!}}a_ih_{i-k}$. Hence, for $k\le n+1$,
\begin{equation}\label{eq:eps_n_bound}
\epsilon_\sigma(n) = \sum_{i=n+1}^\infty a^2_i \le \frac{(n+1-k)!}{(n+1)!}\sum_{i=n+1}^\infty \frac{i!}{(i-k)!} a^2_i \le \frac{(n+1-k)!}{(n+1)!}\left\|\sigma^{(k)}\right\|^2    
\end{equation}

\begin{lemma}\label{lem:eps_n_lip}
    For any $L$-Lipschitz $\sigma$ we have $\epsilon_\sigma(n) \le \frac{L^2}{n}$.
\end{lemma}
\begin{proof}
By \eqref{eq:eps_n_bound} for $k=1$ we get
\[
\epsilon_\sigma(n) \le \frac{1}{n+1}\left\|\sigma'\right\|^2 \le \frac{L^2}{n+1}
\]    
\end{proof}

\begin{lemma}\label{lem:eps_n_sig}
    For the sigmoid activation $\sigma(x) = \int_0^x e^{-\frac{t^2}{2}}dt$ we have $\epsilon_\sigma(n) \le 2^{-n}$.
\end{lemma}
\begin{proof}
We have $\sigma^{(k)}(x) = (-1)^{k-1}\sqrt{(k-1)!}h_{k-1}(x)e^{-\frac{x^2}{2}}$. Indeed, it is not hard to verify it for $k=1$ and $k=2$. For $k > 2$ we have via induction that
\begin{eqnarray*}
\sigma^{(k+1)}(x) &=& (-1)^{k-1}\sqrt{(k-1)!}\left[h'_{k-1}(x) -  x h_{k-1}(x) \right]e^{-\frac{x^2}{2}}    
\\
&\stackrel{\eqref{eq:hermite_derivative}}{=}& (-1)^{k}\sqrt{k!}\frac{1}{\sqrt{k}}\left[x h_{k-1}(x) - \sqrt{k-1}h_{k-2}(x)   \right]e^{-\frac{x^2}{2}}
\\
&\stackrel{\eqref{eq:Hermite}}{=}& (-1)^{k}\sqrt{k!}h_k(x)e^{-\frac{x^2}{2}}
\end{eqnarray*}
Hence, $|\sigma^{(k)}(x)| \le |\sqrt{(k-1)!}h_{k-1}(x)|$, and now \eqref{eq:eps_n_bound} implies that for any $k\le n+1$
\[
\epsilon_\sigma(n)\le \frac{(n+1-k)!}{(n+1)!}(k-1)! = \frac{(n+1-k)!k!}{(n+1)!k} = \frac{1}{k\binom{n+1}{k}}
\]
Taking $k = \left\lceil\frac{n+1}{2} \right\rceil$ we conclude that $\epsilon_\sigma(n)\le 2^{-n}$.
\end{proof}

\section{Conclusion and future work}

One of the prominent approaches for explaining the success of neural networks is trying to show that they are capable of learning complex and ``deep" models. So far this approach has relatively limited success. Despite that significant progress has been made to show that neural networks can learn shallow models, so far, neural networks were shown to learn only ``toy" deep models (e.g. \cite{ghorbani2019limitations, allen2019can, daniely2020learning, yehudai2019power}). Not only that, but there are almost no known rich families of deep models that are efficiently learnable by {\em some} algorithm (not necessarily gradient methods on neural networks). Our paper suggests that random neural networks might be candidate models. To take this approach further, a natural next step, and a central open question that arises from our work, is to show the existence of an algorithm that learns random networks in time that is polynomial both in $\frac{1}{\epsilon}$ and the network size. This question is already open for depth-two ReLU networks with two hidden neurons. We note that as implied by \cite{yehudai2019power}, such a result, even for a single neuron, will have to go beyond polynomial approximation of the network, and even more generally, beyond 
%linear methods such as 
kernel methods.

Our result requires a lower bound $D$ for the network's width, where $D$ is a constant. We conjecture that this requirement can be relaxed, and leave it to future work. 
Additional open directions are (i) the analysis of random convolutional networks, (ii) achieving time and sample complexity of $(\bar d)^{O(\epsilon^{-2})}$ for random networks of any constant depth (and not only for depth two), and (iii) finding a PTAS for random networks of depth $\omega(1)$.

\subsection*{Acknowledgements}

The research described in this paper was funded by the European Research Council (ERC) under the European Union’s Horizon 2022 research and innovation program (grant agreement No. 101041711), and the Israel Science Foundation (grant number 2258/19). 
%Part of 
This research 
%and 
was done as part of the NSF-Simons Sponsored Collaboration on the Theoretical Foundations of Deep Learning.

\bibliography{bib,bib2,bib3}

\begin{thebibliography}{34}
\providecommand{\natexlab}[1]{#1}
\providecommand{\url}[1]{\texttt{#1}}
\expandafter\ifx\csname urlstyle\endcsname\relax
  \providecommand{\doi}[1]{doi: #1}\else
  \providecommand{\doi}{doi: \begingroup \urlstyle{rm}\Url}\fi

\bibitem[Agarwal et~al.(2021)Agarwal, Awasthi, and Kale]{agarwal2021deep}
Naman Agarwal, Pranjal Awasthi, and Satyen Kale.
\newblock A deep conditioning treatment of neural networks.
\newblock In \emph{Algorithmic Learning Theory}, pages 249--305. PMLR, 2021.

\bibitem[Allen-Zhu and Li(2019)]{allen2019can}
Zeyuan Allen-Zhu and Yuanzhi Li.
\newblock What can resnet learn efficiently, going beyond kernels?
\newblock \emph{arXiv preprint arXiv:1905.10337}, 2019.

\bibitem[Applebaum et~al.(2010)Applebaum, Barak, and
  Wigderson]{applebaum2010public}
Benny Applebaum, Boaz Barak, and Avi Wigderson.
\newblock Public-key cryptography from different assumptions.
\newblock In \emph{Proceedings of the forty-second ACM symposium on Theory of
  computing}, pages 171--180, 2010.

\bibitem[Awasthi et~al.(2021)Awasthi, Tang, and
  Vijayaraghavan]{awasthi2021efficient}
Pranjal Awasthi, Alex Tang, and Aravindan Vijayaraghavan.
\newblock Efficient algorithms for learning depth-2 neural networks with
  general relu activations.
\newblock \emph{Advances in Neural Information Processing Systems},
  34:\penalty0 13485--13496, 2021.

\bibitem[Bakshi et~al.(2019)Bakshi, Jayaram, and Woodruff]{bakshi2019learning}
Ainesh Bakshi, Rajesh Jayaram, and David~P Woodruff.
\newblock Learning two layer rectified neural networks in polynomial time.
\newblock In \emph{Conference on Learning Theory}, pages 195--268. PMLR, 2019.

\bibitem[Chen et~al.(2022{\natexlab{a}})Chen, Gollakota, Klivans, and
  Meka]{chen2022hardness}
Sitan Chen, Aravind Gollakota, Adam~R Klivans, and Raghu Meka.
\newblock Hardness of noise-free learning for two-hidden-layer neural networks.
\newblock \emph{arXiv preprint arXiv:2202.05258}, 2022{\natexlab{a}}.

\bibitem[Chen et~al.(2022{\natexlab{b}})Chen, Klivans, and
  Meka]{chen2022learning}
Sitan Chen, Adam~R Klivans, and Raghu Meka.
\newblock Learning deep relu networks is fixed-parameter tractable.
\newblock In \emph{2021 IEEE 62nd Annual Symposium on Foundations of Computer
  Science (FOCS)}, pages 696--707. IEEE, 2022{\natexlab{b}}.

\bibitem[Chen et~al.(2023)Chen, Dou, Goel, Klivans, and Meka]{chen2023learning}
Sitan Chen, Zehao Dou, Surbhi Goel, Adam~R Klivans, and Raghu Meka.
\newblock Learning narrow one-hidden-layer relu networks.
\newblock \emph{arXiv preprint arXiv:2304.10524}, 2023.

\bibitem[Daniely(2017)]{daniely2017sgd}
Amit Daniely.
\newblock Sgd learns the conjugate kernel class of the network.
\newblock In \emph{NIPS}, 2017.

\bibitem[Daniely and Malach(2020)]{daniely2020learning}
Amit Daniely and Eran Malach.
\newblock Learning parities with neural networks.
\newblock In \emph{NIPS}, 2020.

\bibitem[Daniely and Shalev-Shwartz(2016)]{daniely2016complexity}
Amit Daniely and Shai Shalev-Shwartz.
\newblock Complexity theoretic limitations on learning dnf’s.
\newblock In \emph{Conference on Learning Theory}, pages 815--830, 2016.

\bibitem[Daniely and Vardi(2020)]{daniely2020hardness}
Amit Daniely and Gal Vardi.
\newblock Hardness of learning neural networks with natural weightss.
\newblock In \emph{NIPS}, 2020.

\bibitem[Daniely and Vardi(2021)]{daniely2021local}
Amit Daniely and Gal Vardi.
\newblock From local pseudorandom generators to hardness of learning.
\newblock In \emph{Conference on Learning Theory}, pages 1358--1394. PMLR,
  2021.

\bibitem[Daniely et~al.(2016)Daniely, Frostig, and Singer]{daniely2016toward}
Amit Daniely, Roy Frostig, and Yoram Singer.
\newblock Toward deeper understanding of neural networks: The power of
  initialization and a dual view on expressivity.
\newblock In \emph{NIPS}, 2016.

\bibitem[Daniely et~al.(2023)Daniely, Srebro, and
  Vardi]{daniely2023computational}
Amit Daniely, Nathan Srebro, and Gal Vardi.
\newblock Computational complexity of learning neural networks: Smoothness and
  degeneracy.
\newblock \emph{arXiv preprint arXiv:2302.07426}, 2023.

\bibitem[Das et~al.(2019)Das, Gollapudi, Kumar, and
  Panigrahy]{das2019learnability}
Abhimanyu Das, Sreenivas Gollapudi, Ravi Kumar, and Rina Panigrahy.
\newblock On the learnability of deep random networks.
\newblock \emph{arXiv preprint arXiv:1904.03866}, 2019.

\bibitem[Diakonikolas and Kane(2020)]{diakonikolas2020small}
Ilias Diakonikolas and Daniel~M Kane.
\newblock Small covers for near-zero sets of polynomials and learning latent
  variable models.
\newblock In \emph{2020 IEEE 61st Annual Symposium on Foundations of Computer
  Science (FOCS)}, pages 184--195. IEEE, 2020.

\bibitem[Diakonikolas et~al.(2020)Diakonikolas, Kane, Kontonis, and
  Zarifis]{diakonikolas2020algorithms}
Ilias Diakonikolas, Daniel~M Kane, Vasilis Kontonis, and Nikos Zarifis.
\newblock Algorithms and sq lower bounds for pac learning one-hidden-layer relu
  networks.
\newblock In \emph{Conference on Learning Theory}, pages 1514--1539. PMLR,
  2020.

\bibitem[Ge et~al.(2017)Ge, Lee, and Ma]{ge2017earning}
Rong Ge, Jason~D Lee, and Tengyu Ma.
\newblock Learning one-hidden-layer neural networks with landscape design.
\newblock \emph{arXiv preprint arXiv:1711.00501}, 2017.

\bibitem[Ge et~al.(2018)Ge, Kuditipudi, Li, and Wang]{ge2018learning}
Rong Ge, Rohith Kuditipudi, Zhize Li, and Xiang Wang.
\newblock Learning two-layer neural networks with symmetric inputs.
\newblock \emph{arXiv preprint arXiv:1810.06793}, 2018.

\bibitem[Ghorbani et~al.(2019)Ghorbani, Mei, Misiakiewicz, and
  Montanari]{ghorbani2019limitations}
Behrooz Ghorbani, Song Mei, Theodor Misiakiewicz, and Andrea Montanari.
\newblock Limitations of lazy training of two-layers neural network.
\newblock In \emph{Advances in Neural Information Processing Systems}, pages
  9108--9118, 2019.

\bibitem[Glorot and Bengio(2010)]{glorot2010understanding}
Xavier Glorot and Yoshua Bengio.
\newblock Understanding the difficulty of training deep feedforward neural
  networks.
\newblock In \emph{Proceedings of the thirteenth international conference on
  artificial intelligence and statistics}, pages 249--256. JMLR Workshop and
  Conference Proceedings, 2010.

\bibitem[Goel and Klivans(2019)]{goel2019learning}
Surbhi Goel and Adam~R Klivans.
\newblock Learning neural networks with two nonlinear layers in polynomial
  time.
\newblock In \emph{Conference on Learning Theory}, pages 1470--1499. PMLR,
  2019.

\bibitem[Goel et~al.(2017)Goel, Kanade, Klivans, and Thaler]{goel2017reliably}
Surbhi Goel, Varun Kanade, Adam Klivans, and Justin Thaler.
\newblock Reliably learning the relu in polynomial time.
\newblock In \emph{Conference on Learning Theory}, pages 1004--1042. PMLR,
  2017.

\bibitem[Goel et~al.(2020)Goel, Gollakota, Jin, Karmalkar, and
  Klivans]{goel2020superpolynomial}
Surbhi Goel, Aravind Gollakota, Zhihan Jin, Sushrut Karmalkar, and Adam
  Klivans.
\newblock Superpolynomial lower bounds for learning one-layer neural networks
  using gradient descent.
\newblock \emph{arXiv preprint arXiv:2006.12011}, 2020.

\bibitem[He et~al.(2015)He, Zhang, Ren, and Sun]{he2015delving}
Kaiming He, Xiangyu Zhang, Shaoqing Ren, and Jian Sun.
\newblock Delving deep into rectifiers: Surpassing human-level performance on
  imagenet classification.
\newblock In \emph{Proceedings of the IEEE international conference on computer
  vision}, pages 1026--1034, 2015.

\bibitem[Janzamin et~al.(2015)Janzamin, Sedghi, and
  Anandkumar]{janzamin2015beating}
Majid Janzamin, Hanie Sedghi, and Anima Anandkumar.
\newblock Beating the perils of non-convexity: Guaranteed training of neural
  networks using tensor methods.
\newblock \emph{arXiv preprint arXiv:1506.08473}, 2015.

\bibitem[Klivans and Sherstov(2006)]{KlivansSh06}
Adam~R. Klivans and Alexander~A. Sherstov.
\newblock Cryptographic hardness for learning intersections of halfspaces.
\newblock In \emph{FOCS}, 2006.

\bibitem[O'Donnell(2014)]{o2014analysis}
Ryan O'Donnell.
\newblock \emph{Analysis of boolean functions}.
\newblock Cambridge University Press, 2014.

\bibitem[Vempala and Wilmes(2019)]{vempala2019gradient}
Santosh Vempala and John Wilmes.
\newblock Gradient descent for one-hidden-layer neural networks: Polynomial
  convergence and sq lower bounds.
\newblock In \emph{Conference on Learning Theory}, pages 3115--3117. PMLR,
  2019.

\bibitem[Yehudai and Shamir(2019)]{yehudai2019power}
Gilad Yehudai and Ohad Shamir.
\newblock On the power and limitations of random features for understanding
  neural networks.
\newblock \emph{arXiv preprint arXiv:1904.00687}, 2019.

\bibitem[Zhang et~al.(2019)Zhang, Yu, Wang, and Gu]{zhang2019learning}
Xiao Zhang, Yaodong Yu, Lingxiao Wang, and Quanquan Gu.
\newblock Learning one-hidden-layer relu networks via gradient descent.
\newblock In \emph{The 22nd international conference on artificial intelligence
  and statistics}, pages 1524--1534. PMLR, 2019.

\bibitem[Zhang et~al.(2016)Zhang, Lee, and Jordan]{zhang2016l1}
Yuchen Zhang, Jason~D Lee, and Michael~I Jordan.
\newblock l1-regularized neural networks are improperly learnable in polynomial
  time.
\newblock In \emph{International Conference on Machine Learning}, pages
  993--1001. PMLR, 2016.

\bibitem[Zhong et~al.(2017)Zhong, Song, Jain, Bartlett, and
  Dhillon]{zhong2017recovery}
Kai Zhong, Zhao Song, Prateek Jain, Peter~L Bartlett, and Inderjit~S Dhillon.
\newblock Recovery guarantees for one-hidden-layer neural networks.
\newblock In \emph{International conference on machine learning}, pages
  4140--4149. PMLR, 2017.

\end{thebibliography}

\end{document}